\documentclass[runningheads]{llncs}

\makeatletter
\newcommand{\@chapapp}{\relax}%
\makeatother
\usepackage[numbers,sectionbib]{natbib}
\pdfoutput=1
\usepackage[T1]{fontenc}    
\usepackage{authblk}
\usepackage[utf8]{inputenc} 
\usepackage{url}            
\usepackage{booktabs}       
\usepackage{amsfonts}       
\usepackage{nicefrac}       
\usepackage{microtype}      

\usepackage{amsmath,amssymb,latexsym, amsfonts, amscd, amsthm, xy}
\usepackage{epsfig}
\usepackage{multirow}
\usepackage{bbm}
\usepackage{siunitx}
\usepackage{url}
\usepackage{mdframed}
\usepackage{microtype}
\usepackage{graphicx}
\usepackage{subfig}
\usepackage[title]{appendix}
\def\P{\mathbb{P}}
\def\Q{\mathbb{Q}}

\def\cX{\mathcal{X}}
\def\MMD{{{\sc \rm MMD}}(\mathbb{P}, \mathbb{Q})}
\def\SqMMD{{{\sc \rm MMD}}^2(\mathbb{P}, \mathbb{Q})}

\def\EstSqMMD{\widehat{{\sc \rm MMD}}{}^2(\mathbb{P}, \mathbb{Q})}
\def\EstSqMMDIW{\widehat{{\sc \rm MMD}}{}^2_{IW}(\mathbb{P}, \mathbb{Q})}
\def\EstSqMMDMIW{\widehat{{\sc \rm MMD}}{}^2_{MIW}(\mathbb{P}, \mathbb{Q})}
\def\EstSqMMDSNIW{\widehat{{\sc \rm MMD}}{}^2_{SNIW}(\mathbb{P}, \mathbb{Q})}

\def\GenEst{\widehat{\mbox{D}}(\mathbb{P},\mathbb{Q})}
\def\GenEstSNIW{\widehat{\mbox{D}}_{SNIW}(\mathbb{P},\mathbb{Q})}
\def\GenEstIW{\widehat{\mbox{D}}_{IW}(\mathbb{P},\mathbb{Q})}

\def\EE{{\mathbb{E}}}
\def\PP{{\mathbb{P}}}

\def\ModFunc{M}
\def\ThinFunc{T}
\def\weight{w}

\def\scdots{\!\cdot\!\cdot\!\cdot\!}
\def\sdots{...\,}

\usepackage{mathtools}
\DeclarePairedDelimiter\floor{\lfloor}{\rfloor}

\usepackage{times}
\usepackage{soul}
\usepackage{url}
\usepackage[hidelinks]{hyperref}
\hypersetup{
  pdfauthor={...},
  pdftitle={...},
  pdfsubject={...},
  urlcolor=blue,
}

\usepackage{graphicx}
\usepackage[misc]{ifsym}

\begin{document}
\title{Importance Weighted Generative Networks}

\author{Maurice~Diesendruck \inst{1}[\Letter] \and
Ethan R.~Elenberg\inst{2} \and
Rajat Sen\inst{3} \and
Guy W.~Cole\inst{1} \and
Sanjay Shakkottai\inst{1} \and
Sinead A.~Williamson\inst{1}\inst{4}}

\authorrunning{M. Diesendruck et al.}

\institute{The University of Texas at Austin, USA\thanks{R. Sen and S. Shakkottai were partially supported by ARO grant W911NF-17-1-0359} \and ASAPP, Inc.\thanks{Work done primarily while at UT Austin.} \and Amazon, Inc. \and CognitiveScale\\
\email{$\{$momod, rajat.sen, guywcole$\}$@utexas.edu, elenberg@asapp.com, shakkott@austin.utexas.edu, sinead.williamson@mccombs.utexas.edu}}

\maketitle              
\begin{abstract}

While deep generative networks can simulate from complex data distributions, their utility can be hindered by limitations on the data available for training. Specifically, the training data distribution may differ from the target sampling distribution due to sample selection bias, or because the training data comes from a different but related distribution. We present methods to accommodate this difference via \textit{importance weighting}, which allow us to estimate a loss function with respect to a target distribution \textit{even if we cannot access that distribution directly}. These estimators, which differentially weight the contribution of data to the loss function, offer theoretical guarantees that heuristic approaches lack, while giving impressive empirical performance in a variety of settings.

\end{abstract}
\section{Introduction}
Deep generative models have important applications in many fields: we can automatically generate illustrations for text \cite{ZhaXiLu2017}; simulate video streams \cite{VonPirTor2016} or  molecular fingerprints \cite{KadAliKaz2017}; and create privacy-preserving versions of medical time-series data \cite{EstHylRat2017}. Such models use a neural network to parametrize a function $G(Z)$, which maps random noise $Z$ to a target probability distribution $\P$. This is achieved by minimizing a loss function between simulations and data, which is equivalent to learning a distribution over simulations that is indistinguishable from $\P$ under an appropriate two-sample test. In this paper we focus on Generative Adversarial Networks (GANs) \cite{GooPouMir2014,ArjChiBot2017,Binkowski2018,Li2017mmd}, which incorporate an adversarially learned neural network in the loss function; however the results are also applicable to non-adversarial networks \cite{DziRoyGha2015,LiSweZem2015}. 

An interesting challenge arises when we do not have direct access to i.i.d.\ samples from $\P$. This could arise either because observations are obtained via a biased sampling mechanism  \cite{BolChaZou2016,ZhaWanYat2017}, or in a transfer learning setting where our target distribution differs from our training distribution. As an example of the former, a dataset of faces generated as part of a university project may contain disproportionately many young adult faces relative to the population. As an example of the latter, a Canadian hospital system might want to customize simulations to its population while still leveraging a training set of patients from the United States (which has a different statistical distribution of medical records). In both cases, and more generally, we want to generate data from a target distribution $\P$ but only have access to representative samples from a \textit{modified} distribution $\ModFunc\P$. We give a pictorial example of this setting in Figure~\ref{fig:1d_thinned_logistic}.

\begin{figure*}[h]
  \centering
  \subfloat[ Target distribution $\P$]{{\includegraphics[width=.2\textwidth]{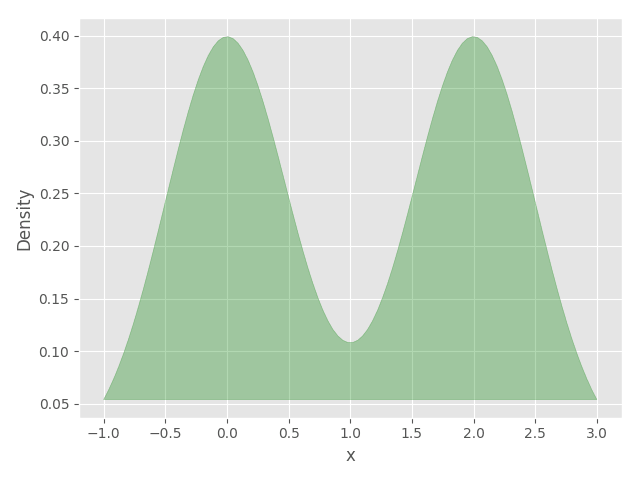}}\label{fig:1d_target_distribution}}
  \qquad
  \subfloat[Observed distribution $\ModFunc\P$ and samples from $\ModFunc\P$]{{\includegraphics[width=.2\textwidth]{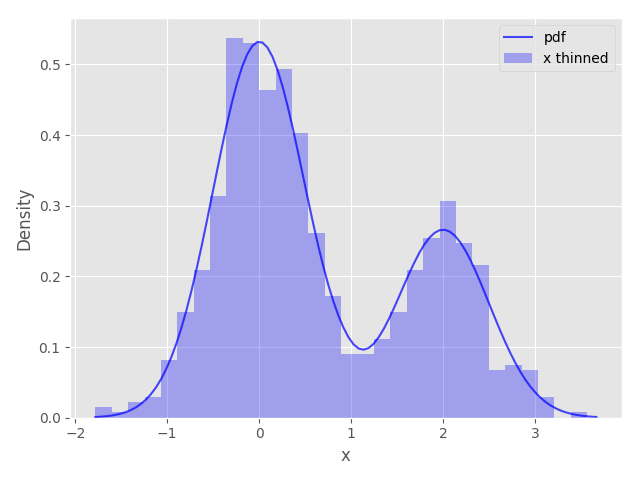}}\label{fig:1d_observed}} \qquad
  \subfloat[Simulations using a standard estimator]{{\includegraphics[width=.2\textwidth]{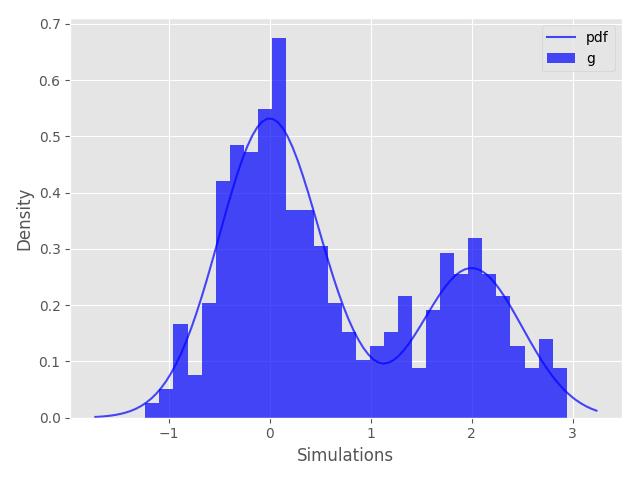}}\label{fig:1d_mmd}}
  \qquad
  \subfloat[Simulations using an importance weighted estimator]{{\includegraphics[width=.2\textwidth]{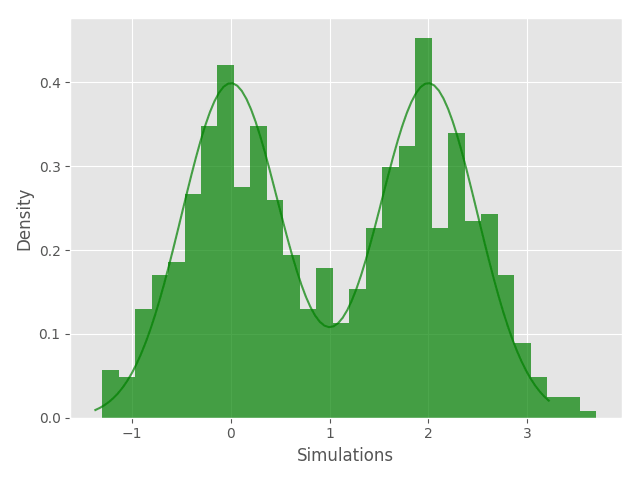}}\label{fig:1d_wmmd}}
  \caption{If our target distribution $\P$ differs from our observed distribution $\ModFunc\P$, using the standard estimator will replicate $\ModFunc\P$, while an importance weighted estimator can replicate the target $\P$. 
  }
  \label{fig:1d_thinned_logistic}
\end{figure*}

In some cases, we can approach this problem using existing methods. For example, if we can reduce our problem to a conditional data-generating mechanism, we can employ Conditional Generative Adversarial Networks (C-GANs) or related models \cite{Mehdi2014cgan,Odena2017}, which enable conditional sampling given one or more latent variables.  However, this requires that  $\ModFunc$ can be described on a low-dimensional space, and that we can sample from our target distribution over that latent space. Further, C-GANs rely on a large, labeled dataset of training samples with diversity over the conditioning variable (within each batch), which becomes a challenge when conditioning on a high-dimensional variable. For example, if we wish to modify a distribution over faces with respect to age, gender and hair length, there may be few exemplars of 80-year-old men with long hair with which to learn the corresponding conditional distribution.

In this paper, we propose an alternate approach based on importance sampling \cite{Owen2013}. Our method modifies an existing GAN by rescaling the observed data distribution $\ModFunc\P$ during training, or equivalently by reweighting the contribution of each data point to the loss function. When training a GAN with samples from $\ModFunc\P$, the standard estimator equally weights the contribution of each point, yielding an estimator of the loss with respect to 
$\ModFunc\P$ and corresponding simulations, as shown in Fig.~\ref{fig:1d_observed} and Fig.~\ref{fig:1d_mmd}. This is not ideal.

In order to yield the desired estimator with respect to our target distribution $\P$, we modify the estimator by reweighting the loss function evaluation for each sample. When the Radon-Nikodym derivative between the target and observed distributions (aka the modifier function $\ModFunc$) is known, we inversely scale each evaluation by that derivative, yielding the finite-sample importance sampling transform on the estimate, which we call the \textit{importance weighted} estimator. This reweighting asymptotically ensures that discrimination, and the corresponding GAN update, occurs with respect to $\P$ instead of $\ModFunc\P$, as shown in Fig.~\ref{fig:1d_target_distribution} and Fig.~\ref{fig:1d_wmmd}.

This approach has multiple advantages and extensions. First, if $\ModFunc$ is known, we can estimate importance weighted losses using robust estimators like the median-of-means estimator, which is crucial for controlling variance in settings where the modifier function $\ModFunc$ has a large dynamic range.  Second, even when the modifier function is only known up to a scaling factor, we can construct an alternative estimator using self-normalized sampling \cite{RobCas2004,Owen2013} to use this partial information, while still maintaining asymptotic correctness. Finally and importantly, for the common case of an unknown modifier function, we demonstrate techniques for estimating it from partially labeled data.

Our contributions are as follows: 1) We provide a novel application of traditional importance weighting to deep generative models. This has connections to many types of GAN loss functions through the theory of U-statistics. 2) We propose several variants of our importance weighting framework for different practical scenarios. When dealing with particularly difficult functions $\ModFunc$, we propose to use robust median-of-means estimation and show that it has similar theoretical guarantees under weaker assumptions, \textit{i.e.} bounded second moment. When $\ModFunc$ is not known fully (only up to a scaling factor), we propose a self-normalized estimator. 3) We conduct an extensive experimental evaluation of the proposed methods on both synthetic and real-world datasets. This includes estimating $\ModFunc$ when less than $4\%$ of the data is labeled with user-provided exemplars.

\subsection{Related Work}\label{sec:related}

Our method aims to generate samples from a distribution $\P$, given access to samples from $\ModFunc\P$. While to the best of our knowledge this has not been explicitly addressed in the GAN literature, several approaches have related goals.

\vspace{1mm}
\noindent \textbf{Domain adaptation:}
Our formulation is related to but distinct from the problem of Domain Adaptation (DA). The challenge of DA is, “If I train on one distribution and test on another, how do I maximize performance on test data?” Critically, the test data is available and extensively used. Instead, our method solves the problem, “Given only a training data distribution, how do I generate from arbitrarily modified versions of it?” The former uses two datasets -- one source and one target -- while the latter uses one dataset and accommodates an arbitrary number of targets. The methodologies are inherently different because the information available is different.

Typical approaches to DA involve finding domain-invariant feature representations for both source and target data. Blitzer, Pereira, Ben-David, and Daume \cite{blitzer2007biographies,ben2007analysis,daume2010frustratingly} write extensively on techniques involving feature correlation and mutual information within classification settings. Pan, Huang, and Gong \cite{pan2008transfer,pan2011domain,huang2007correcting,gong2012geodesic} propose methods with similar goals that find kernel representations under which source and target distributions are close. The work of \cite{huang2007correcting} and \cite{sugiyama2008direct} address covariate shift using kernel-based and importance-weighted techniques, but still inhabit a different setting from our problem since they perform estimation on specific source and target datasets.

Recently, the term DA has been used in the context of adversarially-trained image-to-image translation and downstream transfer learning tasks \cite{isola2017image,taigman2016unsupervised,zhu2017unpaired,hoffman2017cycada}. Typically the goal is to produce representations of the same image in both source and target domains. Such problems begin with datasets from both domains, whereas our setting presents only one source dataset and seeks to generate samples from a hypothetical, user-described target domain.

\vspace{1mm}
\noindent \textbf{Inverse probability weighting:}
Inverse probability weighting (IPW), originally proposed by \cite{horvitz1952} and still in wide use in the field of survey statistics \cite{mansournia2016}, can be seen as a special case of importance sampling. IPW is a weighting scheme used to correct for biased treatment assignment methods in survey sampling. In such settings, the target distribution is known and the sampling distribution is typically finite and discrete, and can easily be estimated from data.

\vspace{1mm}
\noindent \textbf{Conditional GANs:}
Conditional GANs (C-GANs) are an extension of GANs that aim to simulate from a conditional distribution, given some covariate. In the case where our modifier function $\ModFunc$ can be represented in terms of a low-dimensional covariate space, and if we can generate samples from the marginal distribution of $\ModFunc\P$ on that space, then we can, in theory, use a C-GAN to generate samples from $\P$, by conditioning on the sampled covariates.  This strategy suffers from two limitations. First, it assumes we can express $\ModFunc$ in terms of a sampleable distribution on a low-dimensional covariate space. For settings where $\ModFunc$ varies across many data dimensions or across a high-dimensional latent embedding, this ability to sample becomes untenable. Second, learning a family of conditional distributions is typically more difficult than learning a single joint distribution. As we show in our experiments, C-GANs often fail if there are too few real exemplars for a given covariate setting.

Related to C-GANs, \cite{csaba2019domain} proposes conditional generation and a classifier for assigning samples to specific discriminators. While not mentioned, such a structure could feasibly be used to preferentially sample certain modes, if a correspondence between latent features and numbered modes were known.

\vspace{1mm}
\noindent \textbf{Weighted loss:}
In the context of domain adaptation for data with discrete class labels, the strategy of reweighting the Maximum Mean Discrepancy (MMD) \cite{GreBorRas2012} based on class probabilities has been proposed by \cite{yan2017mind}. This approach, however, differs from ours in several ways: It is limited to class imbalance problems, as opposed to changes in continuous-valued latent features; it requires access to the non-conforming target dataset; it provides no theoretical guarantees about the weighted estimator; and it is not in the generative model setting.

\vspace{1mm}
\noindent \textbf{Other uses of importance weights in GANs:}
The language and use of importance weights is not unique to this application, and has been used for other purposes within the GAN context. In \cite{hjelm2017boundary}, for example, importance weights are used to provide policy gradients for GANs in a discrete-data setting. Our application is different in that our target distribution is not that of our data, as it is in \cite{hjelm2017boundary}. Instead we view our data as having been modified, and use importance weights to simulate closer to the hypothetical and desired \textit{unmodified} distribution.

\section{Problem Formulation and Technical Approach}\label{sec:prob-approach}

\noindent \textbf{The problem:} Given training samples from a distribution $\ModFunc \P,$ our goal is to construct (train) a generator function $G(\cdot)$ that produces i.i.d. samples from a distribution $\P.$

To train $G(\cdot)$, we follow the methodology of a Generative Adversarial Network (GAN) \cite{GooPouMir2014}. In brief, a GAN consists of a pair of interacting and evolving neural networks -- a generator neural network with outputs that approximate the desired distribution, and a discriminator neural network that distinguishes between increasingly realistic outputs from the generator and samples from a training dataset.

The loss function is a critical feature of the GAN discriminator, and evaluates the closeness between the samples of the generator and those of the training data. Designing good loss functions remains an active area of research \cite{ArjChiBot2017,Li2017mmd}. One popular loss function is the Maximum Mean Discrepancy (MMD) \cite{GreBorRas2012}, a distributional distance that is zero if and only if the two distributions are the same. As such, MMD can be used to prevent mode collapse~\cite{SalGooZar2016,CheLiJac2017} during training.

\vspace{1mm}
\noindent \textbf{Our approach:} We are able to train a GAN to generate samples from $\P$ using a simple reweighting modification to the MMD loss function. Reweighting forces the loss function to apply greater penalties in areas of the support where the target and observed distributions differ most. 

Below, we formally describe the MMD loss function, and describe its importance weighted variants.

\vspace{1mm}
\noindent \textit{Remark 1 (Extension to other losses). } While this paper focuses on the MMD loss, we note that the above estimators can be extended to any estimator that can be expressed as the expectation of some function with respect to one or more distributions. This class includes losses such as  squared mean difference between two distributions, cross entropy loss, and autoencoder losses \cite{szekely2013energy,Hoe1948,Mis1947}.  Such losses can be estimated from data using a combination of U-statistics, V-statistics and sample averages. Each of these statistics can be reweighted, in a manner analogous to the treatment described above. We provide more comprehensive details in Table~\ref{tab:IWandSNIWestimators}, and in Section~\ref{sec:synth} we evaluate all three importance weighting techniques as applied to the standard cross entropy GAN objective.

\subsection{Maximum Mean Discrepancy between Two Distributions}\label{sec:MMD}

The MMD projects two distributions $\P$ and $\Q$ into a reproducing kernel Hilbert space (RKHS) $\mathcal{H}$, and evaluates the maximum mean distance between the two projections, \textit{i.e.}
\begin{align*}
\MMD:= \sup_{f\in \mathcal{H}}\left(\mathbf{E}_{X\sim \P}[f(X)] - \mathbf{E}_{Y\sim \Q}[f(Y)]\right).
\end{align*}
If we specify the \textit{kernel mean embedding} $\mu_{\P}$ of $\P$ as
$\mu_{\P} = \int k(x,\cdot) d\P(x)$,
where $k(\cdot,\cdot)$ is the characteristic kernel defining the RKHS, then we can write the square of this distance as
\begin{align}
&\SqMMD = ||\mu_{\P}-\mu_{\Q}||_\mathcal{H}^2 \nonumber\\[1ex] 
&= \mathbb{E}_{X,X' \sim \P}[k(X,X')] + \mathbb{E}_{Y,Y' \sim \Q}[k(Y,Y')] \nonumber \\[1ex]
&- 2 \mathbb{E}_{X \sim \P, Y \sim \Q}[k(X,Y)].
\label{eqn:MMDsq}
\end{align}
In order to be a useful loss function for training a neural network, we must be able to estimate $\SqMMD$ from data, and compute gradients of this estimate with respect to the network parameters. Let $\{x_i\}_n$ be a sample $\{X_1=x_1,\ldots,X_n= x_n\} : X_i \sim \P$, and $\{y_i\}_m$ be a sample $\{Y_1=y_1,\ldots,Y_m=y_m\}: Y_i \sim \Q$. We can construct an unbiased estimator $\EstSqMMD$ of $\SqMMD$ \cite{GreBorRas2012} using these samples as
\begin{align}
 \EstSqMMD &= \textstyle \frac{1}{n(n-1)}\sum_{i\neq j}^n k(x_i,x_j) \nonumber\\[1ex] 
  & \textstyle+ \frac{1}{m(m-1)}\sum_{i\neq j}^m k(y_i,y_j) \nonumber\\[1ex]
  & \textstyle- \frac{2}{nm}\sum_{i=1}^n\sum_{j=1}^m k(x_i,y_j).
  \label{eqn:MMDu}
\end{align}

\subsection{Importance Weighted Estimator for Known \texorpdfstring{$\ModFunc$}{M}}\label{sec:weighted_estimator}

We begin with the case where $\ModFunc$ (which relates the distribution of the samples and the desired distribution; formally the Radon-Nikodym derivative) is known. Here, the reweighting of our loss function can be framed as an \textit{importance sampling} problem: we want to estimate $\SqMMD$,  which is in terms of the target distribution $\P$ and the distribution $\Q$ implied by our generator, but we have samples from the modified $\ModFunc\P$. Importance sampling \cite{Owen2013} provides a method for constructing an estimator for the expectation of a function $\phi(X)$ with respect to a distribution $\mathbb{P}$, by taking an appropriately weighted sum of evaluations of $\phi$ at values sampled from a different distribution. We can therefore modify the estimator in \eqref{eqn:MMDu} by weighting each term in the estimator involving data point $x_i$ using the likelihood ratio $\P(x_i)/\ModFunc(x_i)\P(x_i) = 1/\ModFunc(x_i)$, yielding an unbiased importance weighted estimator that takes the form
\begin{align}
 \EstSqMMDIW &= \textstyle \frac{1}{n(n-1)} \sum_{i\neq j}^{n}\frac{k(x_i,x_j)}{\ModFunc(x_i)\ModFunc(x_j)} \nonumber\\[1ex]
 &+ \textstyle \frac{1}{m(m-1)}\sum_{i\neq j}^{m} k(y_i,y_j) \nonumber\\[1ex]
 &- \textstyle \frac{2}{nm}\sum_{i=1}^n \sum_{j=1}^m \frac{k(x_i,y_j)}{\ModFunc(x_i)}. \label{eq:iwmmd}
\end{align}

While importance weighting using the likelihood ratio yields an unbiased estimator \eqref{eq:iwmmd}, the estimator may not concentrate well because the weights $\{1/M(x_i)\}_n$ may be large or even unbounded. We now provide a concentration bound for the estimator in~\eqref{eq:iwmmd} for the case where weights $\{1/M(x_i)\}_n$ are upper-bounded by some maximum value. Note that weights are also lower-bounded above zero so that distributions maintain the same support.

\begin{theorem}\label{thm:IW}
Let $\EstSqMMDIW$ be the unbiased, importance weighted estimator for $\SqMMD$ defined in \eqref{eq:iwmmd}, given $m$ i.i.d samples from $M\mathbb{P}$ and $\mathbb{Q}$, and maximum kernel value K. Further assume that $0 \leq 1/M(x) \leq W$ for all $x \in \cX$. Then
\begin{align*}
 &\PP \textstyle \left(\EstSqMMDIW - \SqMMD > t\right) \leq C, \nonumber \\
 &\text{where } \textstyle C = \exp ((-2t^2m_2)/(K^2(W+1)^4)) \nonumber \\
 & \hspace{1cm} \textstyle m_2 := \floor{\nicefrac{m}{2}}
\end{align*}
\end{theorem}

These guarantees are based on estimator guarantees in \cite{GreBorRas2012}, which in turn build on classical results by Hoeffding \cite{Hoe1963,Hoe1948}. 
We defer the proof of this theorem to Appendix~\ref{sec:proof1}.

\subsection{Robust Importance Weighted Estimator for Known \texorpdfstring{$\ModFunc$}{M}}\label{sec:robust_estimator}

Theorem~\ref{thm:IW} is sufficient to guarantee good concentration of our importance weighted estimator only when $1/M(x)$ is uniformly bounded by some constant $W$, which is not too large. Many class imbalance problems fall into this setting. However, $1/M(x)$ may be unbounded in practice. Therefore, we now introduce a different estimator, which enjoys good concentration even when only $\EE_{X \sim M\P}[1/M(X)^2]$ is bounded, while $1/M(x)$ may be unbounded for many values of $x$. 

The estimator is based on the classical idea of median of means~\cite{MOM_nemirovskii1983problem,MOM_jerrum1986random,MOM_alon1996space,lerasle2018monk}\footnote{\cite{lerasle2018monk} appeared concurrently and contains a different approach for the unweighted estimator. Comparisons are left for future work.}.
Given $m$ samples from $M\P$ and $\Q$, we divide these samples uniformly at random into $k$ equal sized groups, indexed $\{(1),...,(k)\}$. Let $\EstSqMMDIW^{(i)}$ be the value obtained when the estimator in \eqref{eq:iwmmd} is applied on the $i$-th group of samples. Then our median of means based estimator is given by
\begin{align}
   &\EstSqMMDMIW = \operatorname{median} \textstyle \left\{\EstSqMMDIW^{(1)}, \ldots, \EstSqMMDIW^{(k)}  \right\} . \label{eq:mommmd}
\end{align}
\begin{theorem}
\label{thm:MOM} Let $\EstSqMMDMIW$ be the asymptotically unbiased median of means estimator defined in \eqref{eq:mommmd} using $k = mt^2/(8K^2\sigma^2)$ groups. Further assume that $n\!=\!m$ and let $W_2 = \EE_{X \sim M\P}[\nicefrac{1}{M(X)^2}]$ be bounded. Then
\begin{align*}
  &\textstyle \PP \left( \lvert \EstSqMMDMIW  - \SqMMD \rvert > t \right) \leq C, \nonumber \\
  &\text{where } \textstyle C = \exp ((-mt^2)/(64K^2\sigma^2)) \nonumber \\
  & \hspace{1cm} \textstyle \sigma^2 = O\left( W_2^2 + \mbox{MMD}^4(\P,\Q) \right).
\end{align*}
\end{theorem}

We defer the proof of this theorem to Appendix~\ref{sec:proof2}.
Note that the confidence bound in Theorem~\ref{thm:MOM} depends on the term $W_2$ being bounded. This is the second moment of $1/M(X)$ where $X \sim M\P$. Thus, unlike in Theorem~\ref{thm:IW},  this confidence bound may still hold even if $1/M(x)$ is \textit{not uniformly bounded}. When $1/M(X)$ is heavy-tailed with finite variance, \textit{e.g.} Pareto ($\alpha > 2$) or log-normal, then Theorem~\ref{thm:MOM} is valid but Theorem~\ref{thm:IW} does not apply. 

In addition to increased robustness, the median of means MMD estimator is more computationally efficient: since calculating $\EstSqMMDIW$ scales quadratically in the batch size, using the median of means estimator introduces a speed-up that is linear in the number of groups.

\subsection{Self-normalized Importance Weights for Unknown \texorpdfstring{$\ModFunc$}{M}}\label{sec:self_weighted_estimator}

To specify $\ModFunc$, we must know the forms of our target and observed distributions along any marginals where the two differ. In some settings this is available: consider for example a class rebalancing setting where we have class labels and a desired class ratio, and can estimate the observed class ratio from data. This, however, may be infeasible if $\ModFunc$ is continuous and/or varies over several dimensions, particularly if data are arriving in a streaming manner. In such a setting it may be easier to specify a {\em thinning function $\ThinFunc$ that is proportional to $\ModFunc$,} \textit{i.e.} $\ModFunc\P = \frac{\ThinFunc\P}{Z}$ for some unknown $Z$, than to estimate $\ModFunc$ directly. This is because $\ThinFunc$ can be directly obtained from an estimate of how much a given location is underestimated, without any knowledge of the underlying distribution. 

This setting---where the $1/\ModFunc$ weights used in Section~\ref{sec:weighted_estimator} are only known up to a normalizing constant---motivates the use of a \textit{self-normalized} importance sampling scheme, where the weights $w_i \propto \frac{\P(x_i)}{\ModFunc(x_i)\P(x_i)}=\frac{Z}{\ThinFunc(x_i)}$ are normalized to sum to one \cite{RobCas2004,Owen2013}. 
For example, by letting $w_i = \frac{1}{\ThinFunc(x_i)}$, the resulting self-normalized estimator for the squared MMD takes the form
\begin{align}
 \EstSqMMDIW &= \textstyle \frac{\sum_{i\neq j}^n w_i w_jk(x_i,x_j)}{\sum_{i\neq j}^n w_i w_j} \nonumber\\ 
 &+ \textstyle \sum_{i\neq j}^m \frac{k(y_i,y_j)}{m(m-1)} \nonumber\\   
 &- \textstyle 2 \frac{\sum_{i=1}^n\sum_{j=1}^m w_i k(x_i,y_j)}{m\sum_{i=1}^n w_i}.
  \label{eqn:WMMDb}
\end{align}
While use of self-normalized weights means this self-normalized estimator is biased, it is asymptotically unbiased, with the bias decreasing at a rate of $1/n$ \cite{Kon1992}. 
Although we have motivated self-normalized weights out of necessity, in practice they often trade off bias for reduced variance, making them preferable in some practical applications \cite{Owen2013}. 

\vspace{1mm}
\noindent \textit{Remark 2 (Boundedness of density ratios). }  In Theorem~\ref{thm:IW}, the weights $\{1/M(x_i)\}_n$ represent a bounded density ratio $\P/\ModFunc\P$.  For bounded distributions that are strictly positive everywhere, density ratios have finite bounds. For unbounded distributions, both cases exist. For example, ratios of Laplace distributions are bounded everywhere, while ratios of Gaussian distributions are unbounded in the tails --- without loss of generality:
\begin{align}
    &(\mbox{Laplace}) && \lim_{x\rightarrow \infty} \frac{\exp(-|x - \mu|)}{\exp(-|x|)} = \lim_{x\rightarrow \infty} \exp(|\mu|) < \infty\\
    &(\mbox{Gaussian}) && \lim_{x\rightarrow \infty} \frac{\exp(-(x - \mu)^2)}{\exp(-x^2)} = \lim_{x\rightarrow \infty} \exp(2x\mu + \mu^2) = \infty.
\end{align}
For any unnormalized modifier function $T$ bounded between positive constants $\theta_L$ and $\theta_H$, the density ratio $\P/\ModFunc\P = 1/M$ is bounded on $(0,\theta_H/\theta_L]$. To show this consider the following:
Let the weight $W = \frac{\P}{\ModFunc\P} = \frac{\P}{\ThinFunc\P/Z} = \frac{Z}{\ThinFunc} = \frac{\int \ThinFunc(x)p(x)dx}{\ThinFunc(x)} = \frac{1}{\ModFunc}$, where $Z$ is the normalizing constant for unnormalized $\ThinFunc\P$. To upper-bound this quantity, consider the following two bounds. The numerator $Z = \int \ThinFunc(x)p(x)dx = \int[p(x)\theta_H + p(x)(\ThinFunc(x) - \theta_H)]dx \leq \theta_H$, since the first term is $\theta_H$, and the second term is less than or equal to zero. The denominator $\ThinFunc(x) \geq \theta_L$, by definition. Together, the weight $W = 1/\ModFunc$ is bounded above by $\theta_H/\theta_L$.
\vspace{1mm}

More generally, in addition to not knowing the normalizing constant $Z$, we might also not know the thinning function $\ThinFunc$. For example, $\ThinFunc$ might vary along some latent dimension---perhaps we want to have more images of people fitting a certain aesthetic, rather than corresponding to a certain observed covariate or class. In this setting, a practitioner may be able to estimate  $\ThinFunc(x_i)$, or equivalently $\weight_i$, for a small number of training points $x_i$, by considering how much those training points are under- or over-represented. Continuous-valued latent preferences can therefore be expressed by applying higher weights to points deemed more appealing. From here, we can use function estimation techniques, such as neural network regression, to estimate $\ThinFunc$ from a small number of labeled data points. 

\subsection{Approximate Importance Weighting by Data Duplication}
In the importance weighting scheme described above, each data point is assigned a weight $1/M(x_i)$. We can obtain an approximation to this method by including $\lceil 1/M(x_i)\rceil$ duplicates of data point $x_i$ in our training set. We refer to this approach as \textit{importance duplication}. Importance duplication obviously introduces discretization errors, and if our estimator is a U-statistic it will introduce bias (\textit{e.g.}\ in the MMD example, if two or more copies of the data point $x_i$ appear in a minibatch, then $k(x_i,x_i)$ will appear in the first term of~\eqref{eqn:MMDu}). However, as we show in the experimental setting, even though this approach lacks theoretical guarantees it provides generally good performance.

Data duplication can be done as a pre-processing step, making it an appealing choice if we have an existing GAN implementation that we do not wish to modify. In other settings, it is less appealing, since duplicating data adds an additional step and increases the amount of data the algorithm must process. Further, if we were to use this approximation in a setting where $M$ is unknown, we would have to perform this data duplication on the fly as our estimate of $M$ changes.

\begin{table*}[h!]
\caption{Constructing importance weighted estimators for losses involving U-statistics, V-statistics and sample averages. Here, $\mathcal{U}$ is the set of all $r$-tuples of numbers from 1 to $n$ without repeats, and $\mathcal{V}$ is the set of $r$-tuples allowing repeats. Below, let $X_{u,*} = X_{u_1},\sdots,X_{u_r}$.}
\centering 
\begin{tabular}{llll }
\toprule
& $\GenEst$ & $\GenEstIW$ & $\GenEstSNIW$ \\
\midrule
U-statistic &   $\displaystyle \frac{1}{^nP_r}\sum_{u\in\mathcal{U}}g(X_{u,*})$ & $\displaystyle\frac{1}{^nP_r}\sum_{u\in\mathcal{U}}\frac{g(X_{u,*})}{\ModFunc(X_{u_1})\scdots\ModFunc(X_{u_r})}$ &
$\displaystyle\frac{\sum_{u\in\mathcal{U}}w_{u_1}\scdots w_{u_r} g(X_{u,*})}{\sum_{u\in\mathcal{U}}w_{u_1}\scdots w_{u_r}}$\\
V-statistic &$\displaystyle\frac{1}{n^r}\sum_{v\in \mathcal{V}}g(X_{v,*})$ &$\displaystyle\frac{1}{n^r}\sum_{v\in \mathcal{V}}\frac{g(X_{v,*})}{\ModFunc(X_{v_1})\scdots\ModFunc(X_{v_r})}$ & $\displaystyle\frac{\sum_{v\in \mathcal{V}} w_{v_1}\scdots w_{v_r}g(X_{v,*})}{\sum_{v_r=1}^n w_{v_1}\scdots w_{v_r}}$ \\
Average & $\displaystyle\frac{1}{nm}\sum_{i=1}^n\sum_{j=1}^m f(X_i,Y_j)$ &$\displaystyle\frac{1}{nm}\sum_{i=1}^n\sum_{j=1}^m \frac{f(X_i,Y_j)}{\ModFunc(X_i)}$ & $\displaystyle\frac{\sum_{i=1}^n w_i\sum_{j=1}^m f(X_i,Y_j)}{m\sum_{i=1}^n w_i}$\\
\bottomrule
\end{tabular}
\label{tab:IWandSNIWestimators}
\end{table*}

\section{Evaluation}\label{sec:eval}

In this section, we show that our estimators, in conjunction with an appropriate generator network, allow us to generate simulations that are close in distribution to our target distribution, even when we only have access to this distribution via a biased sampling mechanism. Further, we show that our method performs comparably with, or better than, conditional GAN baselines.

Most of our weighted GAN models are based on the MMD-GAN of~\cite{Li2017mmd}, replacing the original MMD loss with either our importance weighted loss $\EstSqMMDIW$ (IW-MMD), our median of means loss $\EstSqMMDMIW$ (MIW-MMD), or our self-normalized loss $\EstSqMMDSNIW$ (SNIW-MMD). We also use a standard MMD loss with an importance duplicated dataset (ID-MMD). Other losses used in \cite{Li2017mmd} are also appropriately weighted, following the form in Table~\ref{tab:IWandSNIWestimators}. In the synthetic data examples of Section~\ref{sec:synth}, the kernel is a fixed radial basis function, while in all other sections it is adversarially trained using a discriminator network as in~\cite{Li2017mmd}. 

To demonstrate that our method is applicable to other losses, in Section~\ref{sec:synth} we also create models that use the standard cross entropy GAN loss, replacing this loss with either an importance weighted estimator (IW-CE), a median of means estimator (MIW-CE) or a self-normalized estimator (SNIW-CE). We also combine a standard cross entropy loss with an importance duplicated dataset (ID-CE). These models used a two-layer feedforward neural network with ten nodes per layer.

Where appropriate, we compare against a conditional GAN (C-GAN). If $\ModFunc$ is known exactly and expressible in terms of a lower-dimensional covariate space, a conditional GAN (C-GAN) offers an alternative method to sample from $\P$: learn the appropriate conditional distributions given each covariate value, sample new covariate values, and then sample from $\P$ using each conditional distribution. 

\subsection{Can GANs with Importance Weighted Estimators Recover Target Distributions, Given \texorpdfstring{$\ModFunc$}{M}?}\label{sec:synth}

To evaluate whether using importance weighted estimators can recover target distributions, we consider a synthetically generated distribution that has been manipulated along a latent dimension. Under the target distribution, a latent representation $\theta_i$ of each data point lives in a ten-dimensional space, with each dimension independently Uniform(0,1). The observed data points $x_i$ are then obtained as $\theta_i^T F$, where $F_{ij} \sim \mathcal{N} (0,1)$  represents a fixed mapping between the latent space and $D$-dimensional observation space. In the training data, the first dimension of $\theta_i$ has distribution $p(\theta)=2\theta,0<\theta\leq 1$. We assume that the modifying function $M(x_i)=2\theta_{i,1}$ is observed, but that the remaining latent dimensions are unobserved.

In our experiments, we generate samples from the target distribution using each of the methods described above, and include weighted versions of the cross entropy GAN to demonstrate that importance weighting can be generalized to other losses.

To compare methods, we report the empirically estimated KL divergence between the target and generated samples in Table~\ref{tab:low_d_kl}. 
Similar results using squared MMD and energy distance are shown in Table~\ref{tab:low_d_scores} and Table~\ref{tab:low_d_scores_best} in Appendix~\ref{sec:moreExp}.
For varying real dimensions $D$, importance weighted methods outperform C-GAN under a variety of measures. 

In some instances C-GAN performs well in two dimensions, but deteriorates quickly as the problem becomes more challenging with higher dimensions. We also note that many runs of C-GAN either ran into numerical issues or diverged; in these cases we report the best score among runs, before training failure.

\begin{table*}[h!]
  \caption[Distributional discrepancies between generated and target data samples]{Estimated KL divergence between generated and target samples (mean $\pm$ standard deviation over 20 runs).}
  \label{tab:low_d_kl}
  \centering
  \begin{tabular}{lccc}
    \toprule
    \cmidrule(r){1-4}
    Model                    & 2D                 & 4D                 & 10D   \\
    \midrule
    IW-CE      & 0.1768 $\pm$ 0.0635  & 0.4934 $\pm$ 0.1238  & 2.7945 $\pm$ 0.5966  \\
    MIW-CE     & 0.3265 $\pm$ 0.1071  & 0.6251 $\pm$ 0.1343  & 3.3093 $\pm$ 0.7179  \\
    SNIW-CE    & 0.0925 $\pm$ 0.0272  & 0.3864 $\pm$ 0.1478  & 2.3060 $\pm$ 0.6915  \\
    ID-CE      & 0.1526 $\pm$ 0.0332  & 0.3444 $\pm$ 0.0766  & 1.4128 $\pm$ 0.3288  \\
    IW-MMD     & \textbf{0.0343 $\pm$ 0.0230}  & \textbf{0.0037 $\pm$ 0.0489}  & \textbf{0.5133 $\pm$ 0.1718}  \\
    MIW-MMD    & 0.2698 $\pm$ 0.0618  & 0.0939 $\pm$ 0.0522  & 0.8501 $\pm$ 0.3271  \\
    SNIW-MMD   & 0.0451 $\pm$ 0.0132  & 0.1435 $\pm$ 0.0377  & 0.6623 $\pm$ 0.0918  \\
    C-GAN          & 0.0879 $\pm$ 0.0405  & 0.3108 $\pm$ 0.0982  & 6.9016 $\pm$ 2.8406  \\
    \bottomrule
  \end{tabular}
\end{table*}

While the above experiment can be evaluated numerically and provide good results for thinning on a 
continuous-valued variable, it is difficult to visualize the outcome. In order to better visualize whether the target distribution is correctly achieved, we also run experiments with explicit and easily measurable class distributions. In Figure~\ref{fig:multi_dampen}, we show a class rebalancing problem on MNIST digits, where an initial uneven distribution between three classes can be accurately rebalanced.
We also show good performance modifying a balanced distribution to specific boosted levels (see Appendix~\ref{sec:moreExp}).
Together, these experiments provide evidence that importance weighting controls the simulated distribution in the desired way.

\begin{figure}[h!]
\centering
    \subfloat[Source, uneven distribution of 0s, 1s, and 5s]{{\includegraphics[width=.27\columnwidth]{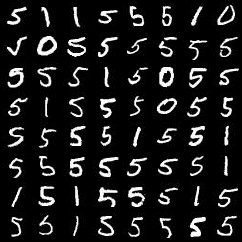} }\label{fig:multi_dampen_x_source}}
  \hspace{.05in}
    \subfloat[Source (left), simulation (right); target of $\nicefrac{1}{3}$-$\nicefrac{1}{3}$-$\nicefrac{1}{3}$]{{\includegraphics[width=.27\columnwidth]{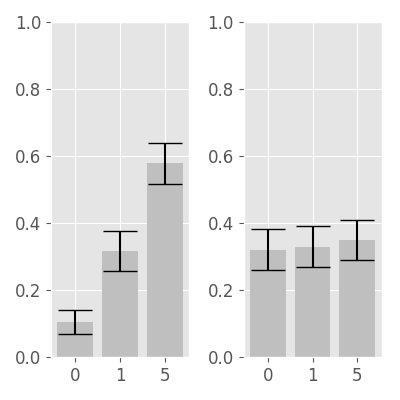} } \label{fig:multi_dampen_plots}}
  \hspace{.05in}
    \subfloat[Simulations, balanced distribution]{{\includegraphics[width=.27\columnwidth]{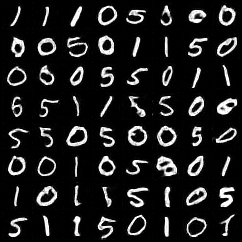} }\label{fig:multi_dampen_result}}
  \caption{Importance weights are used to accurately rebalance an uneven class distribution.}%
 \label{fig:multi_dampen}
\end{figure}

\subsection{In a High-dimensional Image Setting, 
How Does Importance Weighting Compare with Conditional Generation?}\label{sec:yearbook}

\begin{figure}[t!]
  \centering
  \subfloat[Conditional DCGAN]{{\includegraphics[width=.43\columnwidth]{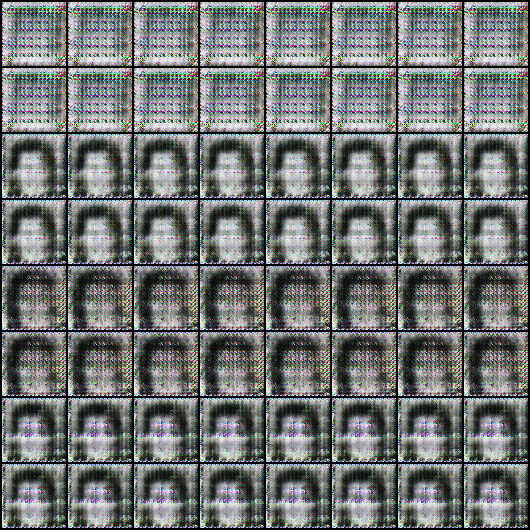} }\label{fig:yearbook6CGAN}}
  \hspace{.005in}
    \subfloat[ID-MMD]{{\includegraphics[width=.43\columnwidth]{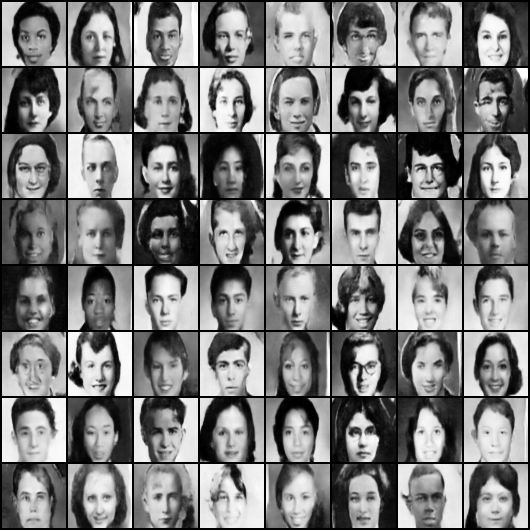} } 
    \label{fig:yearbook6US}}\\
    \subfloat[Importance Weighting (IW-MMD)]{{\includegraphics[width=.43\columnwidth]{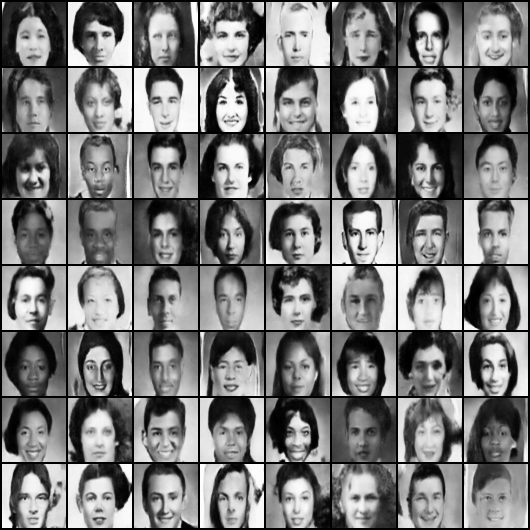} }\label{fig:yearbook6IW}}
  \hspace{.005in}
    \subfloat[Median of Means (MIW-MMD)]{{\includegraphics[width=.43\columnwidth]{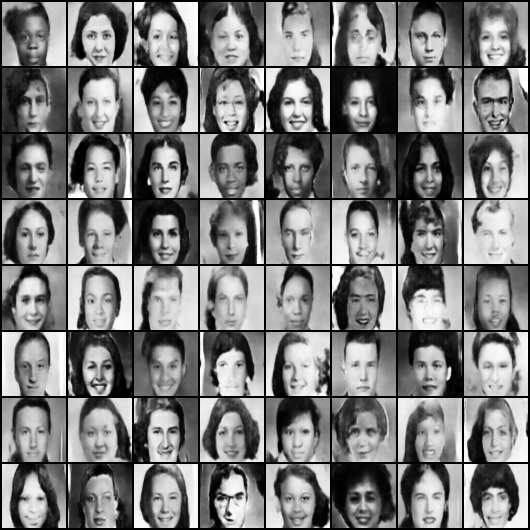} } \label{fig:yearbook6MOM}} 
  \caption{Example generated images for all example networks, Yearbook dataset~\cite{Ginosar2017portraits}. Target distribution is uniform across half-decades, while the training set is unbalanced. 
  }\label{fig:yearbook6}
\end{figure}
Next we evaluate performance of importance weighted MMD on high-dimensional image generation.
In this section we address two questions: Can our estimators generate simulations from $\P$ in such a setting, and how do the resulting images compare with those obtained using a C-GAN? To do so, we evaluate several generative models on the Yearbook dataset~\cite{Ginosar2017portraits}, which contains over $37,\!000$ high school yearbook photos across over $100$ years and demonstrates evolving styles and demographics. The goal is to produce images uniformly across each half decade. Each GAN, however, is trained on the original dataset, which contains many more photos from recent decades.

Since we have specified $\ModFunc$ in terms of a single covariate (time), we can compare with C-GANs. For the C-GAN, we use a conditional version of the standard DCGAN architecture (C-DCGAN) \cite{radford2015unsupervised}. 

Figure~\ref{fig:yearbook6} shows generated images from each network. All networks were trained until convergence. The images show a diversity across hairstyles, demographics and facial expressions, indicating the successful temporal rebalancing. Even while importance duplication introduces approximations and lacks the theoretical guarantees of the other two methods, all three importance-based methods achieve comparable quality. 
Since some covariates have fewer than $65$ images, C-DCGAN cannot learn the conditional distributions, and is unstable across a variety of training parameters.
Implementation details and additional experiments are shown in Appendix~\ref{sec:moreExp}.

\subsection{When \texorpdfstring{$\ModFunc$}{M} Is Unknown, but Can Be Estimated Up to a Normalizing Constant on a Subset of Data, Are We Able to Sample from our Target Distribution?}\label{sec:selfnorm}

In many settings, especially those with high-dimensional latent features, we will not know the functional form of $\ModFunc$, or even the corresponding thinning function $\ThinFunc$. We would still, however, like to be able to express a preference for certain areas of the latent space. To do so, we propose labeling a small subset of data using weights that correspond to preference. To expand those weights to the entire dataset, we train a neural network called the estimated weighting function. This weighting function takes encoded images as input, and outputs continuous-valued weights. Since this function exists in a high-dimensional space that changes as the encoder is updated, and since we do not know the full observed distribution on this space, we are in a setting unsuitable for conditional methods, and therefore use self-normalized estimators (SNIW-MMD).

We evaluate using a collection of sevens from the MNIST dataset, where the goal is to generate more European-style sevens with horizontal bars. Out of 5915 images, 200 were manually labeled with a weight (reciprocal of a thinning function value), where sevens with no horizontal bar were assigned a 1, and sevens with horizontal bars were assigned weights between 2 and 9 based on the width of the bar.
\begin{figure}[h!]
  \centering
  \subfloat[Data]{{\includegraphics[width=.27\columnwidth]{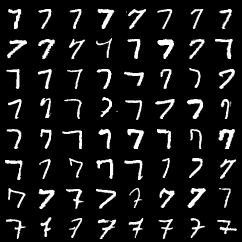} } \label{fig:mnist_sevens_data_sorted}}
  \subfloat[Generator]
  {\hspace{0.05in}
  \includegraphics[width=.27\columnwidth]{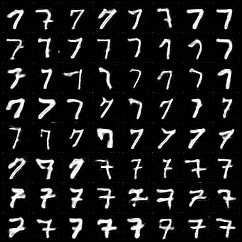} \label{fig:mnist_sevens_gens_sorted}}
 {\hspace{0.05in}
 \subfloat[KS distance]{
  \includegraphics[width=.27\columnwidth]{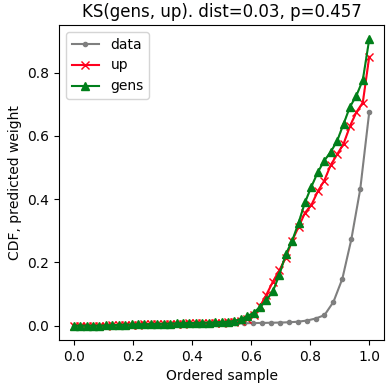} \label{fig:mnist_sevens_lines}}}\\
  \caption{Partial labeling and an importance weighted estimator boost the presence of sevens with horizontal bars. In \ref{fig:mnist_sevens_data_sorted} and \ref{fig:mnist_sevens_gens_sorted}, samples are sorted by predicted weight, and in \ref{fig:mnist_sevens_lines}, the empirical CDFs of data, generated, and importance duplicated draws, are shown, where the latter serves as a theoretical target. The  generated distribution is close in distance to the target.}%
 \label{fig:mnist_sevens}
\end{figure}

Fig.~\ref{fig:mnist_sevens_data_sorted} shows 64 real images, sorted in terms of their predicted weights -- note that the majority have no horizontal bar. Fig.~\ref{fig:mnist_sevens_gens_sorted} shows 64 generated simulations, sorted in the same manner, clearly showing an increase in the number of horizontal-bar sevens. 

To test the quantitative performance, we display and compare the empirical CDFs of weights from simulations, data, and importance duplicated data. For example, if a batch of data $[A, B, C]$ has weights $[1, 3, 2]$, this implies that we expected three times as many $B$-like points and two times as many $C$-like points as $A$-like points. A simulator that achieves this target produces simulations like $[A, B, B, B, C, C]$ with weights $[1, 3, 3, 3, 2, 2]$, equivalent to an importance duplication of data weights. Using importance duplicated weights as a theoretical target, we measure our model's performance by computing the Kolmogorov-Smirnov (KS) distance between CDFs of simulated and importance duplicated weights. Fig.~\ref{fig:mnist_sevens_lines} shows a small distributional distance between simulations and their theoretical target, with $d_{KS} = 0.03$, $p=0.457$. 

\section{Conclusions and Future Work}\label{sec:disc}
We present three estimators for the MMD (and a wide class of other loss functions) between target distribution $\P$ and the distribution $\Q$ implied by our generator. These estimators can be used to train a GAN to simulate from the target distribution $\P$, given samples from a modified distribution $\ModFunc\P$. We present solutions for when $\ModFunc$ is potentially unbounded, is unknown, or is known only up to a scaling factor. 

We demonstrate that importance weighted estimators allow deep generative models to match target distributions for common and challenging cases with continuous-valued, multivariate latent features. This method avoids heuristics while providing good empirical performance and theoretical guarantees.

Though the median of means estimator offers a more robust estimate of the MMD, we may still experience high variance in our estimates, for example if we rarely see data points from a class we want to boost. An interesting future line of research is exploring how variance-reduction techniques \cite{DefBacLac2014} or adaptive batch sizes \cite{DeYadJacGol2017} could be used to overcome this problem.

\bibliographystyle{splncs04}
\bibliography{paper}



\begin{appendices}

\section{Proof of Theorem~\texorpdfstring{\ref{thm:IW}}{\ref*{thm:IW}}}
\label{sec:proof1}

Before we prove Theorem~\ref{thm:IW}, we will define some notation. Suppose $p = \{p_1,...,p_m\}$, $x = \{x_1,...,x_m\}$ and $y = \{y_1,...,y_m\}$ are the empirical samples obtained from $\P$, $\ModFunc\P$ and $\Q$, respectively. We use the following quantity as in \cite{GreBorRas2012}, with samples $p$ and $y$: 
\begin{align}\label{eq:combinedustat}
h(z_i,z_j) = k(p_i,p_j) + k(y_i,y_j) - k(p_i,y_j) - k(p_j,y_i).
\end{align}
Here, $z_i = (p_i,y_i)$ denotes a pair of i.i.d.\ samples from $\P \times \Q$. The estimator $\EstSqMMD$ can be written as \[ \EstSqMMD = \frac{1}{m(m-1)} \sum_{i \neq j} h(z_i,z_j). \]

\begin{proof}
Now consider the setting with samples $x$ and $y$. For positive $W$, and a modifying function $M(\cdot)$ with values on $[1/W, \infty]$, the weights $w(x_i) = 1 / \ModFunc(x_i)$ are therefore bounded as $0 < w(x_i) \leq W$. We rewrite the function $h$, now including weights, as

\begin{align}
h'(z_i, z_j) := w(x_i)w(x_j)k(x_i, x_j) + k(y_i, y_j) - w(x_i)k(x_i, y_j) - w(x_j)k(x_j, y_i) .
\end{align}

Assuming the kernel $k(\cdot,\cdot)$ is bounded between $0$ and $K$, we can infer function bounds such that $-2WK \leq h'(z_i, z_j) \leq K(W^2 + 1)$.

Using Theorem 10 from Gretton et al.~\cite{GreBorRas2012}, we have that
\begin{align}
\begin{split}
P(\EstSqMMDIW - \SqMMD > t) &\leq \exp \left(\frac{-2t^2m_2}{((K(W^2 + 1) - (-2WK))^2} \right)\\
&= \exp \left(\frac{-2t^2m_2}{K^2(W+1)^4} \right),
\end{split}
\end{align}
where $m_2 := \floor{\nicefrac{m}{2}}$, as the MMD requires two samples to evaluate $h(z_i, z_j)$.
\end{proof}

\section{Proof of Theorem~\texorpdfstring{\ref{thm:MOM}}{\ref*{thm:MOM}}}
\label{sec:proof2}
Before we prove Theorem~\ref{thm:MOM}, we prove two functional lemmas.

\begin{lemma}\label{lem:mmdVar}
The variance of the estimator $\EstSqMMDIW$ given $m$ samples each from $\ModFunc \P$ and $\P$ is upper bounded by $2\sigma^2/m$, where $\sigma^2 = \mathrm{Var}(h(Z_i,Z_j))$ and $Z_i\sim\ModFunc \P \times \Q$. 
\end{lemma}
\begin{proof}
Let  $\sigma^2 = \mathrm{Var}(h(Z_i,Z_j))$ and let $\sigma_1^2 = \mathrm{Var}(\EE[h(Z_i,Z_j) | Z_i=z_i])$. Using Hoeffding's Theorem and the 
fact that $2\sigma_1^2 \leq \sigma^2$ \cite{Hoe1948},
we bound the variance of the unbiased MMD U-statistic by
\begin{align*}
\mathrm{Var} (\EstSqMMDMIW) &= \frac{1}{\binom{m}{2}} \sum_{c=1}^2 \binom{2}{c}\binom{m-2}{2-c}\sigma_c^2 \\
&\leq \frac{1}{\binom{m}{2}} \left[2(m-2)\sigma_1^2 + \sigma^2 \right] \\
&\leq \frac{2}{m(m-1)} \left[ (m-1)\sigma^2 \right] = \frac{2\sigma^2}{m} .
\end{align*}
\end{proof}

\begin{lemma}\label{lem:varBound}
We have the following bound:
\begin{align*}
\mathrm{Var} (h(Z_i,Z_j)) \leq 5 \left( K^2 \left( \EE \left[ \frac{1}{\ModFunc(X)^2}\right] + 1 \right)^2  +  \mbox{MMD}^4(\P,\Q) \right),
\end{align*}
where the expectation is with respect to the distribution $M\mathbb{P}$. 
\end{lemma}

\begin{proof}
Let $\mu =  \SqMMD$. Note that $\EE[h(Z_i,Z_j)] = \mu$. Therefore, we have the following chain,
\begin{align*}
&\mathrm{Var} (h(Z_i,Z_j))\\ &= \EE[(h(Z_i,Z_j) - \mu)^2] \\
&= \EE \left[ \left(  \frac{k(X_i,X_j)}{\ModFunc(X_i)\ModFunc(X_j)} + k(Y_i,Y_j) - \frac{k(X_i,Y_j)}{\ModFunc(X_i)} - \frac{k(X_j,Y_i)}{\ModFunc(X_j)} - \mu \right)^2\right] \\
&= 25 \EE \left[ \left(  \frac{k(X_i,X_j)}{5\ModFunc(X_i)\ModFunc(X_j)} + k(Y_i,Y_j)/5 - \frac{k(X_i,Y_j)}{5\ModFunc(X_i)} - \frac{k(X_j,Y_i)}{5\ModFunc(X_j)} - \frac{\mu}{5} \right)^2\right] \\
&\leq 25 \EE \left[ \frac{1}{5}\left(  \frac{k(X_i,X_j)^2}{\ModFunc(X_i)^2\ModFunc(X_j)^2} + k(Y_i,Y_j)^2 + \frac{k(X_i,Y_j)^2}{\ModFunc(X_i)^2} + \frac{k(X_j,Y_i)^2}{\ModFunc(X_j)^2} + \mu^2 \right)\right] \\
& \leq 5 \EE \left[ \frac{K^2}{\ModFunc(X_i)^2\ModFunc(X_j)^2}\right] + 5 K^2 + 10 \EE \left[ \frac{K^2}{\ModFunc(X_i)^2}\right] + 5\mu^2
\end{align*}
This implies the lemma as $X_i,X_j$ are independent and generated from $M\PP$. The first inequality follows from the fact that $(\sum_i p_i a_i)^2 \leq \sum_i p_i a_i^2$, if $p$ lies on the simplex. The last inequality follows from the assumption that $|k(.,.)|\leq K$. 
\end{proof}

\begin{proof}[Proof of Theorem~\ref{thm:MOM}]
Define $\tilde{\sigma}^2$ to be the variance upper bound in Lemma~\ref{lem:varBound}. Suppose we have $m$ samples from $\ModFunc\P$ and $\Q$, $z_i = (x_i,y_i)$ for $i = 1,...,m$. We divide the samples into $k = 8 \log (1/\delta)$ groups, where $\log (1/\delta) = mt^2/64K^2\sigma^2$. We form the estimators of type $\EstSqMMDIW$ for each of the groups indexed $l = 1,...,k$. Let $\EstSqMMDIW^{(l)}$ be the estimator for group $l$. 

Note that by Lemma~\ref{lem:mmdVar} the variance of $\EstSqMMDIW^{(l)}$ is bounded by $2k\tilde{\sigma}^2/m$. Therefore, with probability at least $3/4$, $\EstSqMMDIW^{(l)}$ is within $2 \times \sqrt{2k\tilde{\sigma}^2/m}$ distance of its mean. As such, the probability that the median is not within the distance $2 \times \sqrt{2k\tilde{\sigma}^2/m}$ is at most $\PP(\mathrm{Bin}(k,1/4) > k/2)$, which is exponentially small in $k$. Substituting the value of $k$ yields the result. 
\end{proof}

\section{Implementation and Additional Experiments}\label{sec:moreExp}

\subsection{Synthetic Data}

For the synthetic data experiment of Section~\ref{sec:synth}, we show the full results in Table~\ref{tab:low_d_scores} and Table~\ref{tab:low_d_scores_best} for three discrepancy measures: squared MMD, energy distance, and estimated KL divergence. We note that the squared MMD used in evaluation is the standard estimator.

\begin{table*}[t!]
  \caption[Distributional discrepancies between generated and target data samples]{Squared MMD, energy distance, and estimated KL divergence between generated and target samples (mean $\pm$ standard deviation over 20 runs). Note: Estimated KL divergence is based on \cite{wang2009divergence}.}
  \label{tab:low_d_scores}
  \centering
  \begin{tabular}{lccc}
    \toprule
    \cmidrule(r){1-4}
    Model                    & 2D                 & 4D                 & 10D   \\
    \midrule
    & & \textit{MMD}$^2$\\[1ex]
    IW-CE      & 0.0171 $\pm$ 0.0029  & 0.0214 $\pm$ 0.0030  & 0.0214 $\pm$ 0.0044  \\
    MIW-CE     & 0.0246 $\pm$ 0.0038  & 0.0293 $\pm$ 0.0066  & 0.0233 $\pm$ 0.0036  \\
    SNIW-CE    & 0.0165 $\pm$ 0.0015  & 0.0197 $\pm$ 0.0035  & 0.0186 $\pm$ 0.0035  \\
    ID-CE     & 0.0304 $\pm$ 0.0025  & 0.0230 $\pm$ 0.0019  & 0.0154 $\pm$ 0.0017  \\
    IW-MMD     & 0.0199 $\pm$ 0.0019  & 0.0174 $\pm$ 0.0010  & \textbf{0.0105 $\pm$ 0.0003}  \\
    MIW-MMD    & 0.0586 $\pm$ 0.0038  & 0.0342 $\pm$ 0.0016  & 0.0136 $\pm$ 0.0006  \\
    SNIW-MMD   & \textbf{0.0149 $\pm$ 0.0011}  & \textbf{0.0137 $\pm$ 0.0007}  & 0.0107 $\pm$ 0.0002  \\
    C-GAN          & 0.0174 $\pm$ 0.0040  & 0.0177 $\pm$ 0.0029  & 0.0630 $\pm$ 0.0302  \\
    \midrule
    & & \textit{Energy}\\[1ex]
    IW-CE      & 0.0141 $\pm$ 0.0027  & 0.0361 $\pm$ 0.0044  & 0.0794 $\pm$ 0.0203  \\
    MIW-CE     & 0.0230 $\pm$ 0.0041  & 0.0473 $\pm$ 0.0083  & 0.1040 $\pm$ 0.0188  \\
    SNIW-CE    & 0.0144 $\pm$ 0.0037  & 0.0350 $\pm$ 0.0052  & 0.0720 $\pm$ 0.0080  \\
    ID-CE     & 0.0361 $\pm$ 0.0048  & 0.0600 $\pm$ 0.0073  & 0.0998 $\pm$ 0.0156  \\
    IW-MMD     & 0.0179 $\pm$ 0.0031  & 0.0341 $\pm$ 0.0120  & 0.0700 $\pm$ 0.0274  \\
    MIW-MMD    & 0.0881 $\pm$ 0.0303  & 0.0908 $\pm$ 0.0238  & 0.2123 $\pm$ 0.0893  \\
    SNIW-MMD   & \textbf{0.0136 $\pm$ 0.0020}  & \textbf{0.0291 $\pm$ 0.0055}  & \textbf{0.0506 $\pm$ 0.0147}  \\
    C-GAN          & 0.0140 $\pm$ 0.0057  & 0.0297 $\pm$ 0.0110  & 0.5828 $\pm$ 0.5416  \\
    \midrule
    & & \textit{KL}\\[1ex]
    IW-CE      & 0.1768 $\pm$ 0.0635  & 0.4934 $\pm$ 0.1238  & 2.7945 $\pm$ 0.5966  \\
    MIW-CE     & 0.3265 $\pm$ 0.1071  & 0.6251 $\pm$ 0.1343  & 3.3093 $\pm$ 0.7179  \\
    SNIW-CE    & 0.0925 $\pm$ 0.0272  & 0.3864 $\pm$ 0.1478  & 2.3060 $\pm$ 0.6915  \\
    ID-CE     & 0.1526 $\pm$ 0.0332  & 0.3444 $\pm$ 0.0766  & 1.4128 $\pm$ 0.3288  \\
    IW-MMD     & \textbf{0.0343 $\pm$ 0.0230}  & \textbf{0.0037 $\pm$ 0.0489}  & \textbf{0.5133 $\pm$ 0.1718}  \\
    MIW-MMD    & 0.2698 $\pm$ 0.0618  & 0.0939 $\pm$ 0.0522  & 0.8501 $\pm$ 0.3271  \\
    SNIW-MMD   & 0.0451 $\pm$ 0.0132  & 0.1435 $\pm$ 0.0377  & 0.6623 $\pm$ 0.0918  \\
    C-GAN          & 0.0879 $\pm$ 0.0405  & 0.3108 $\pm$ 0.0982  & 6.9016 $\pm$ 2.8406  \\
    \bottomrule
  \end{tabular}
\end{table*}

\begin{table*}[t!]
  \caption[Minimum distributional discrepancies between generated and target data samples]{Squared MMD, energy distance, and estimated KL divergence between generated and target samples (best over 20 runs). Note: Estimated KL divergence is based on \cite{wang2009divergence}.}
  \label{tab:low_d_scores_best}
  \centering
  \begin{tabular}{lccc}
    \toprule
    \cmidrule(r){1-4}
    Model                    & 2D                 & 4D                 & 10D   \\
    \midrule
    & & \textit{MMD}$^2$\\[1ex]
    IW-CE      & 0.0140  & 0.0175  & 0.0148  \\
    MIW-CE     & 0.0187  & 0.0213  & 0.0157  \\
    SNIW-CE    & 0.0141  & 0.0152  & 0.0138  \\
    ID-CE     & 0.0257  & 0.0198  & 0.0128  \\
    IW-MMD     & 0.0172  & 0.0147  & \textbf{0.0099}  \\
    MIW-MMD    & 0.0522  & 0.0321  & 0.0124  \\
    SNIW-MMD   & 0.0130  & \textbf{0.0125}  & 0.0104  \\
    C-GAN          & \textbf{0.0101}  & 0.0133  & 0.0152  \\
    \midrule
    & & \textit{Energy}\\[1ex]
    IW-CE      & 0.0099  & 0.0281  & 0.0520  \\
    MIW-CE     & 0.0163  & 0.0331  & 0.0659  \\
    SNIW-CE    & 0.0075  & 0.0239  & 0.0584  \\
    ID-CE     & 0.0306  & 0.0476  & 0.0715  \\
    IW-MMD     & 0.0128  & 0.0163  & 0.0294  \\
    MIW-MMD    & 0.0570  & 0.0578  & 0.0824  \\
    SNIW-MMD   & 0.0107  & 0.0220  & \textbf{0.0290}  \\
    C-GAN          & \textbf{0.0061}  & \textbf{0.0155}  & 0.0872  \\
    \midrule
    & & \textit{KL}\\[1ex]
    IW-CE      & 0.0754  & 0.3543  & 1.4763  \\
    MIW-CE     & 0.1534  & 0.4110  & 1.9377  \\
    SNIW-CE    & 0.0378  & 0.1787  & 1.2751  \\
    ID-CE     & 0.088  & 0.2257  & 0.8249  \\
    IW-MMD     & \textbf{-0.0079}  & \textbf{-0.0632}  & \textbf{0.1122}  \\
    MIW-MMD    & 0.2025  & 0.0171  & 0.2811  \\
    SNIW-MMD   & 0.0297  & 0.0733  & 0.4911  \\
    C-GAN          & -0.0043  & 0.1384  & 1.5569  \\
    \bottomrule
  \end{tabular}
\end{table*}

\subsection{Yearbook}

The C-DCGAN is trained for $25$ epochs using the ADAM optimizer with $\alpha=2\,\mathrm{e}\,{-4}$, $\beta_1=0.5$, and $\beta_2=0.999$, and a batch size of $64$. The latent variable has dimension $100$, and we condition on a $22$-dimensional vector corresponding to each half-decade in the dataset.

Networks for the importance weighted and median of means estimator
are trained using and RMSprop optimizer with learning rate $5\,\mathrm{e}\,{-5}$. We use the same regularizers and schedule of generator-discriminator updates as \cite{Li2017mmd}. For $\EstSqMMDIW$ a batch size of $64$ was used, and for $\EstSqMMDMIW$, a large batch of $128$ was split randomly into $8$ groups of $16$ samples.

Figure~\ref{fig:yearbookApp} shows interpolation in the latent $z$ for the half-decade experiment in Section~\ref{sec:yearbook}.
Figure~\ref{fig:yearbook4App} shows another Yearbook experiment with larger imbalance between $2$ time periods: Old (1930) and New (1980-2013).
MMD-GANs are trained for $15,\!500$ generator iterations. 

Figure~\ref{fig:yearbookBinaryCGAN} shows a related experiment in which we produce more older images given a dataset with equal amounts of old (1925-1944) and new (2000-2013) photos. Here, each time period contains over $4,\!500$ images, which increases the stability of conditional GAN training. MMD-GANs are trained until convergence ($8,\!000$ generator iterations).

\begin{figure}[t!]
\centering
    \subfloat[Conditional DCGAN]{{\includegraphics[width=.4\columnwidth]{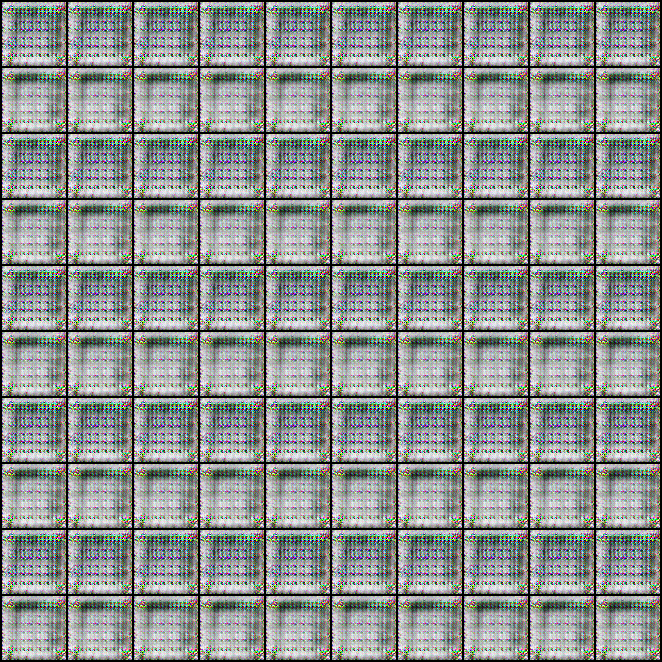} }\label{fig:yearbookBinaryCGANyearbook}}
  \hspace{.05in}
    \subfloat[Importance Duplication]{{\includegraphics[width=.4\columnwidth]{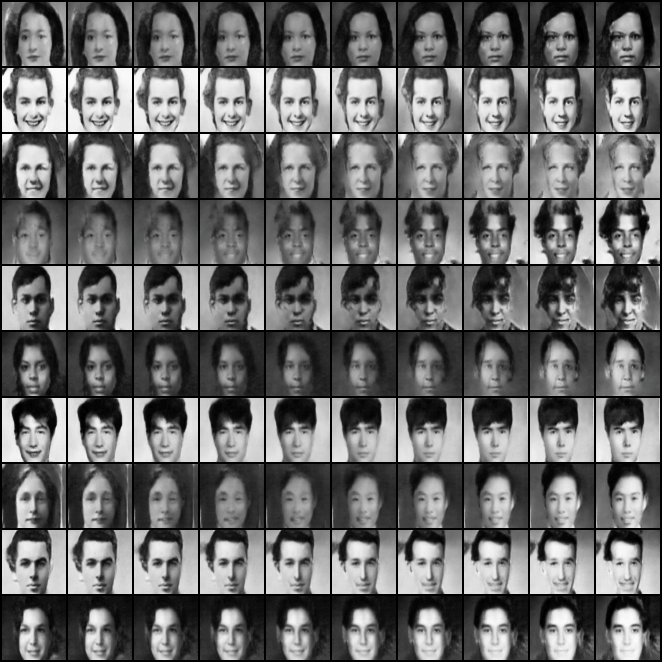} } \label{fig:yearbookUSinterp}}\\
    \subfloat[Importance Weighting (IW-MMD)]{{\includegraphics[width=.4\columnwidth]{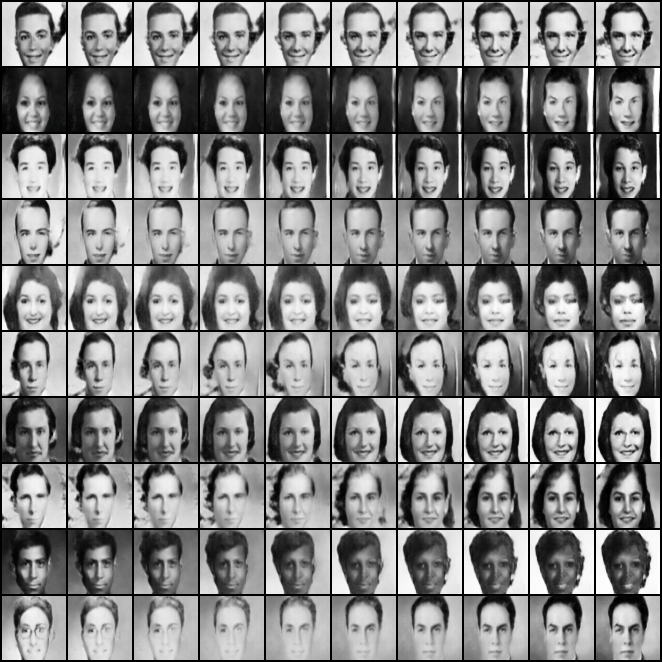} }\label{fig:yearbookIWinterp}}
  \hspace{.05in}
     \subfloat[Median of Means (MIW-MMD)]{{\includegraphics[width=.4\columnwidth]{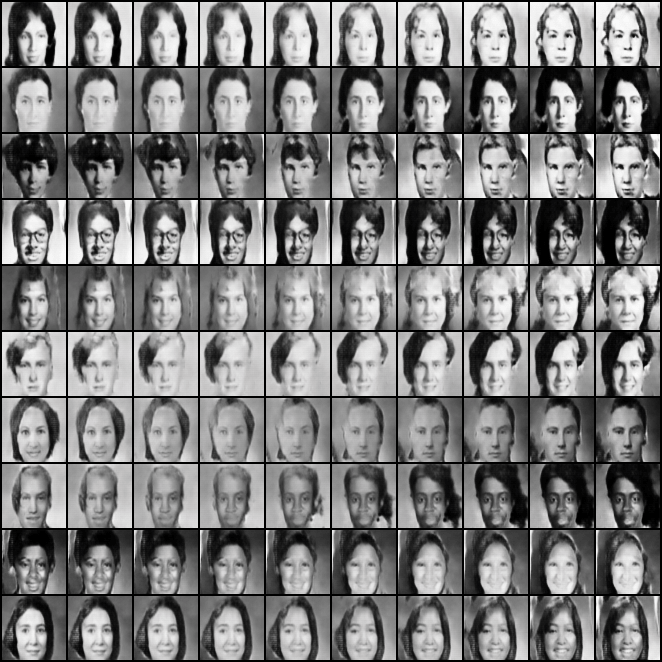} } \label{fig:yearbookMOMinterp}}\\
  \caption{Example interpolations in the latent $z$ space, half-decades experiment.}%
 \label{fig:yearbookApp}
\end{figure}

\begin{figure}
\centering
  \subfloat[Conditional DCGAN]{{\includegraphics[width=.25\columnwidth]{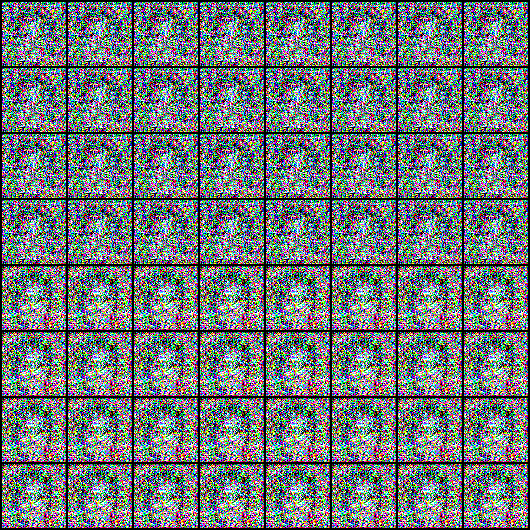} }\label{fig:yearbook4CGAN}}
  \hspace{.05in}
    \subfloat[Importance Duplication]{{\includegraphics[width=.25\columnwidth]{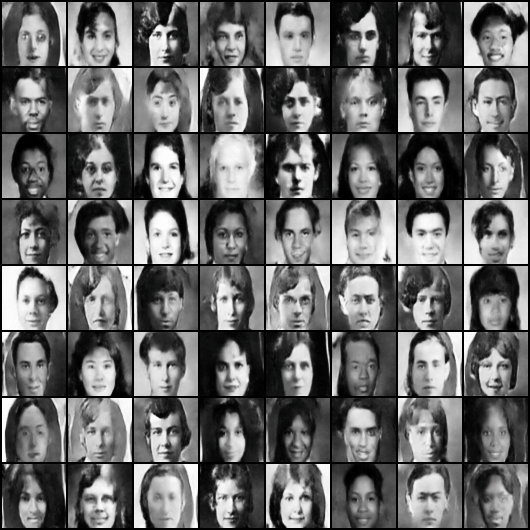} } \label{fig:yearbook4US}}\\
    \subfloat[Importance Weighting]{{\includegraphics[width=.25\columnwidth]{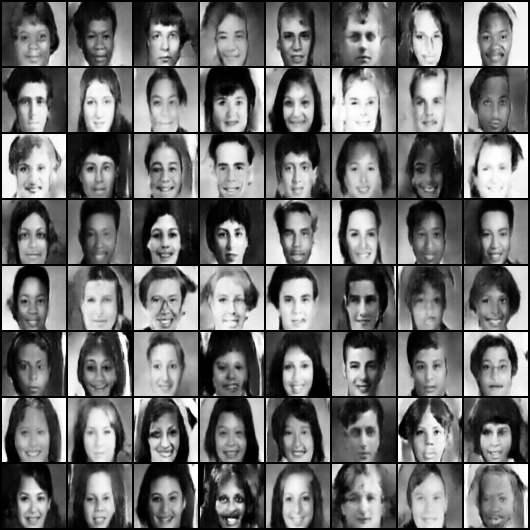} }\label{fig:yearbook4IW}}
  \hspace{.05in}
    \subfloat[Median of Means]{{\includegraphics[width=.25\columnwidth]{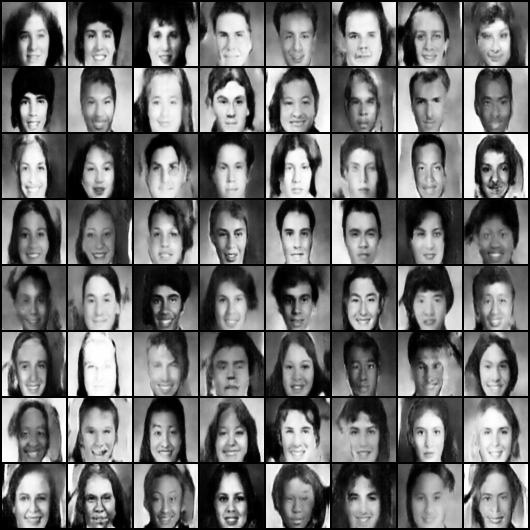} } \label{fig:yearbook4MOM}} \\
 \subfloat[Conditional DCGAN]{{\includegraphics[width=.25\columnwidth]{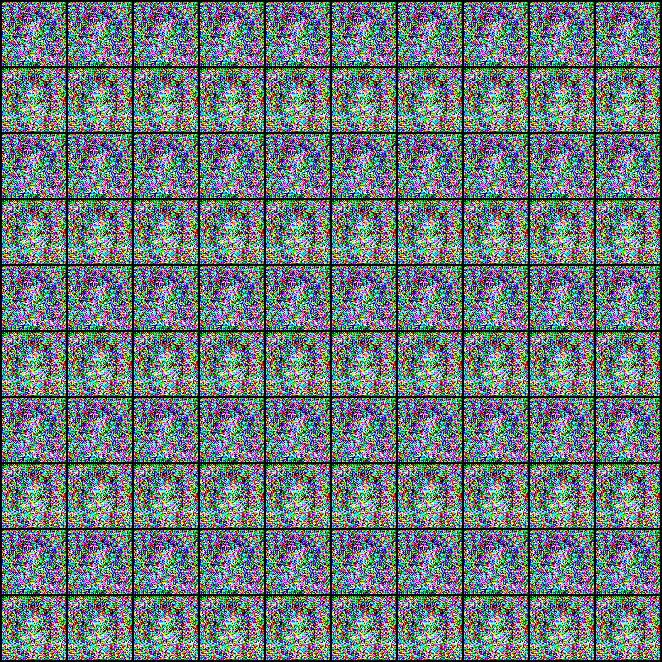} }\label{fig:yearbook4CGANyearbook}}
  \hspace{.05in}
    \subfloat[Importance Duplication]{{\includegraphics[width=.25\columnwidth]{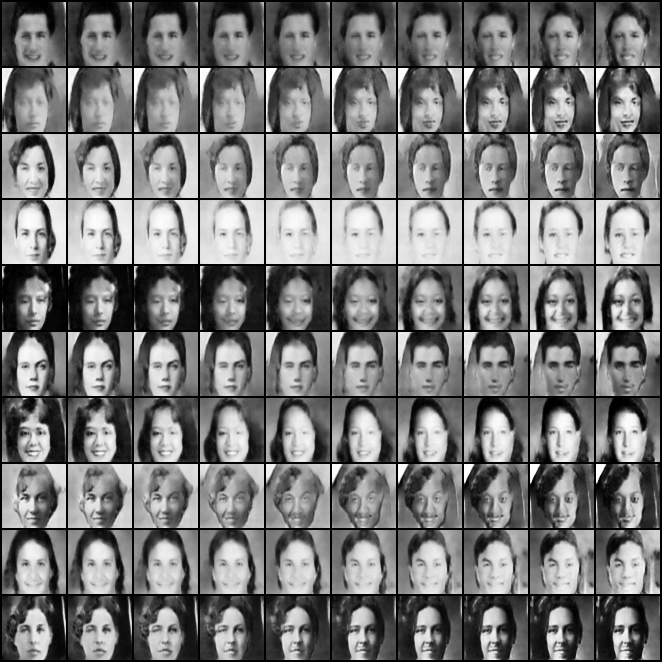} } \label{fig:yearbook4USinterp}}\\
    \subfloat[Importance Weighting]{{\includegraphics[width=.25\columnwidth]{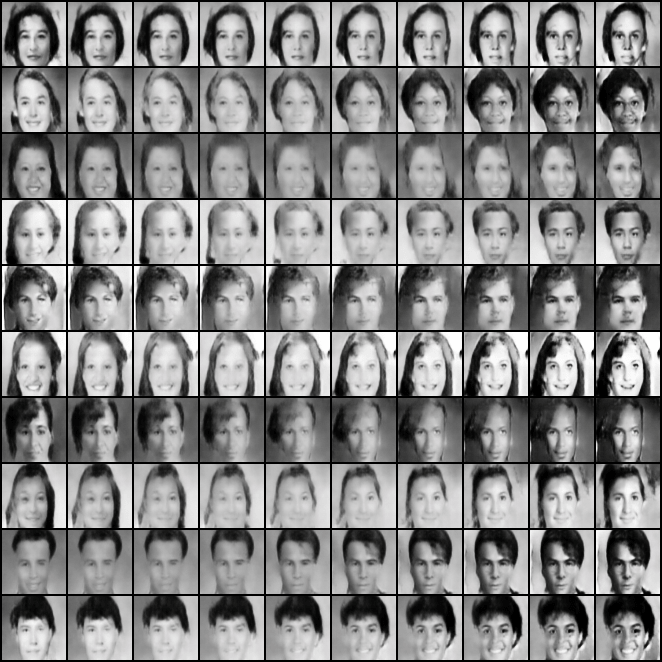} }\label{fig:yearbook4IWinterp}}
  \hspace{.05in}
    \subfloat[Median of Means]{{\includegraphics[width=.25\columnwidth]{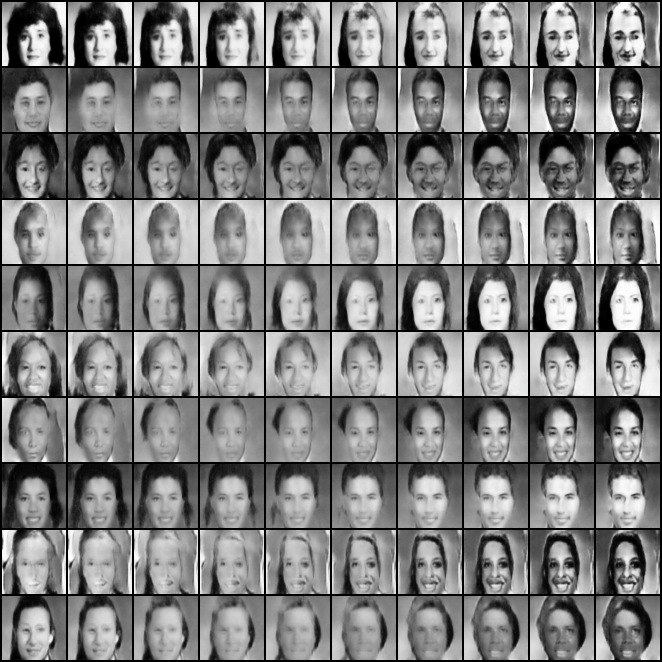} } \label{fig:yearbook4MOMinterp}}
  \caption{Example generated yearbook images from two time periods: Old (1930) and Recent (1980-2013). The target distribution is 50\%/50\%, while the training set is 1\%/99\%. Again, C-DCGAN is unstable across a variety of training parameters, while the importance weighted MMD-GAN methods produce reasonable samples (b)--(d) with meaningful interpolations in the latent space (f)--(h).}
 \label{fig:yearbook4App}
\end{figure}

\begin{figure}[t!]
  \centering
  \subfloat[Conditional DCGAN]{{\includegraphics[width=.25\columnwidth]{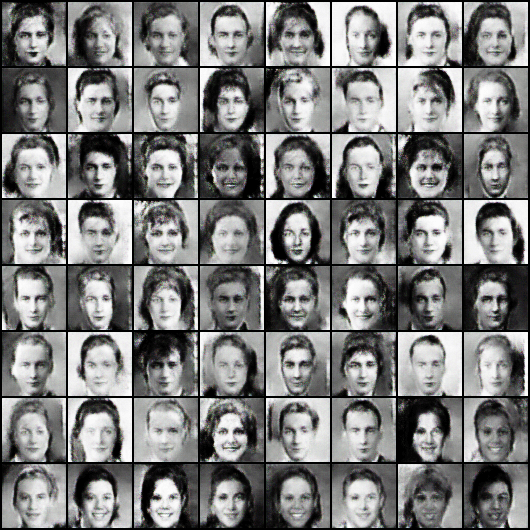} }\label{fig:yearbookBinaryCGANDCGAN}}
  \hspace{.05in}
    \subfloat[Importance Duplication]{{\includegraphics[width=.25\columnwidth]{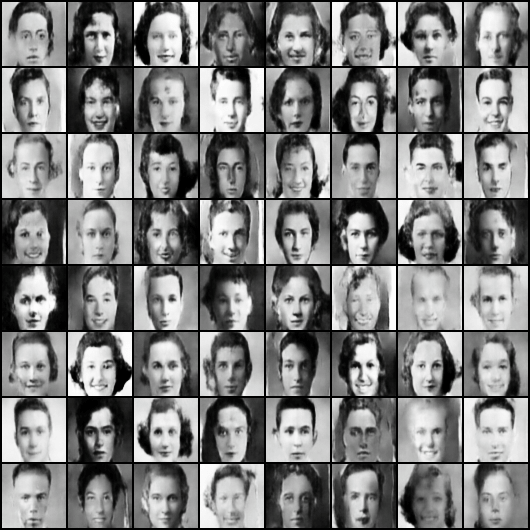} } \label{fig:yearbookBinaryUS}}\\
    \subfloat[Importance Weighting]{{\includegraphics[width=.25\columnwidth]{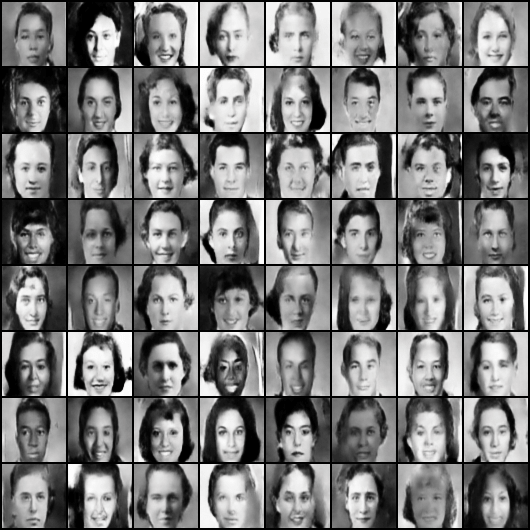} }\label{fig:yearbookBinaryIW}}
    \subfloat[Median of Means]{{\includegraphics[width=.25\columnwidth]{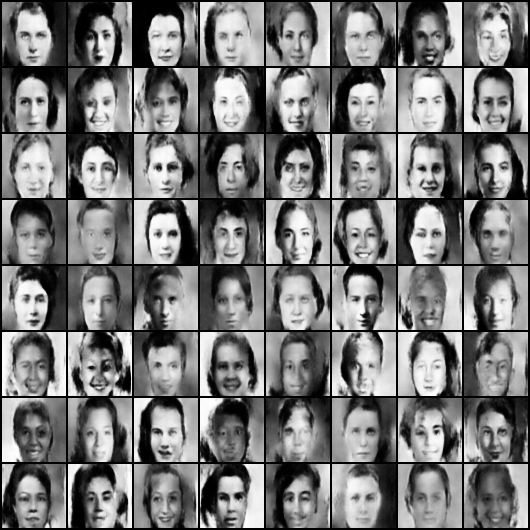} } \label{fig:yearbookBinaryMOM}} \\
    \subfloat[Conditional DCGAN]{{\includegraphics[width=.25\columnwidth]{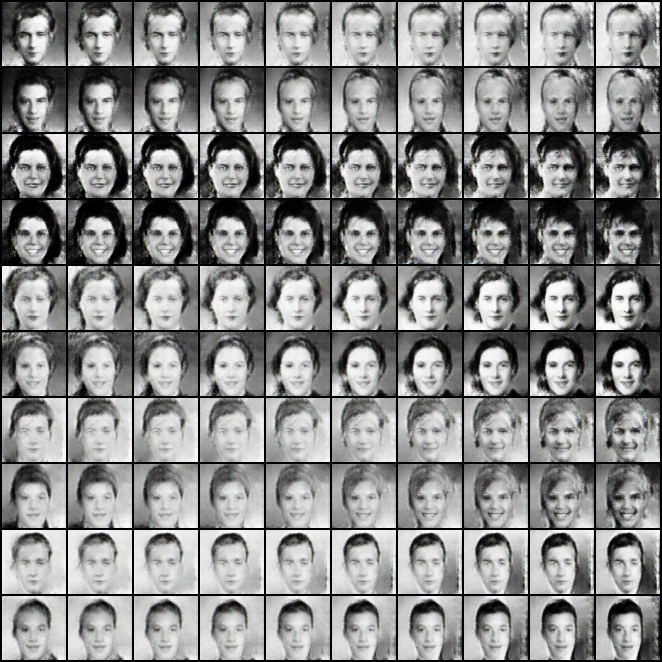} }\label{fig:yearbookBinaryCGANinterp}}
  \hspace{.05in}
    \subfloat[Importance Duplication]{{\includegraphics[width=.25\columnwidth]{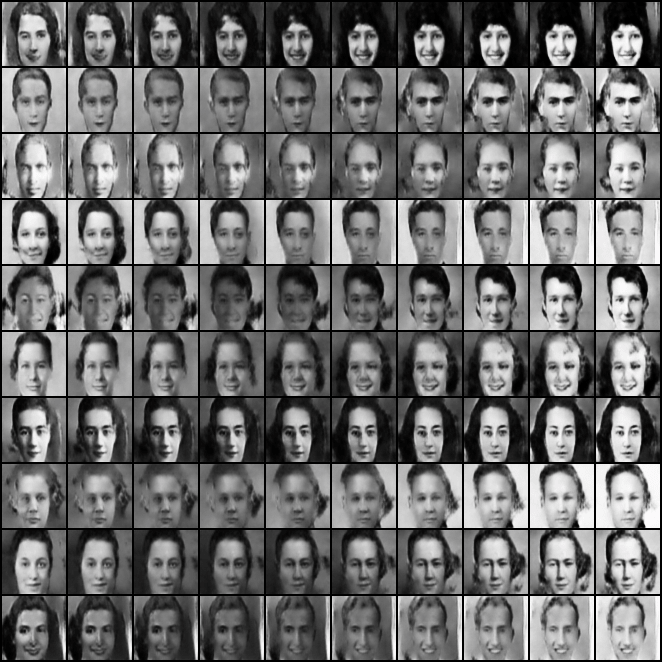} } \label{fig:yearbookBinaryUSinterp}}\\
    \subfloat[Importance Weighting]{{\includegraphics[width=.25\columnwidth]{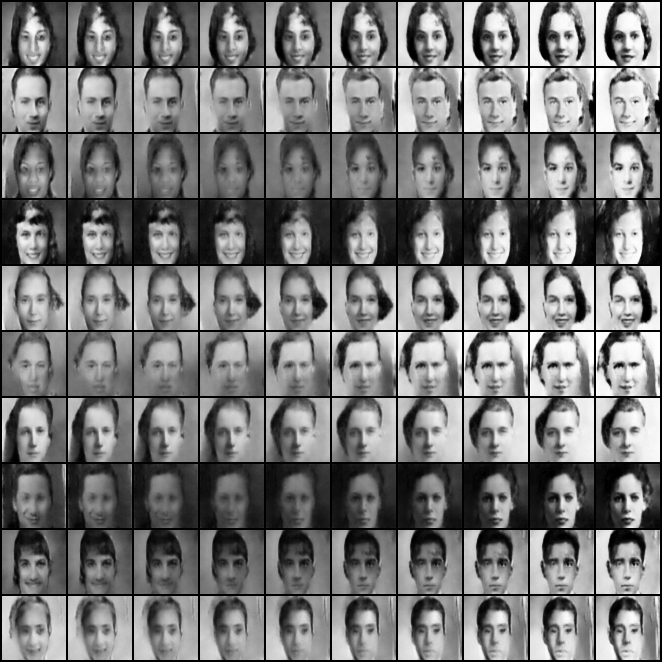} }\label{fig:yearbookBinaryIWinterp}}
  \hspace{.05in}
    \subfloat[Median of Means]{{\includegraphics[width=.25\columnwidth]{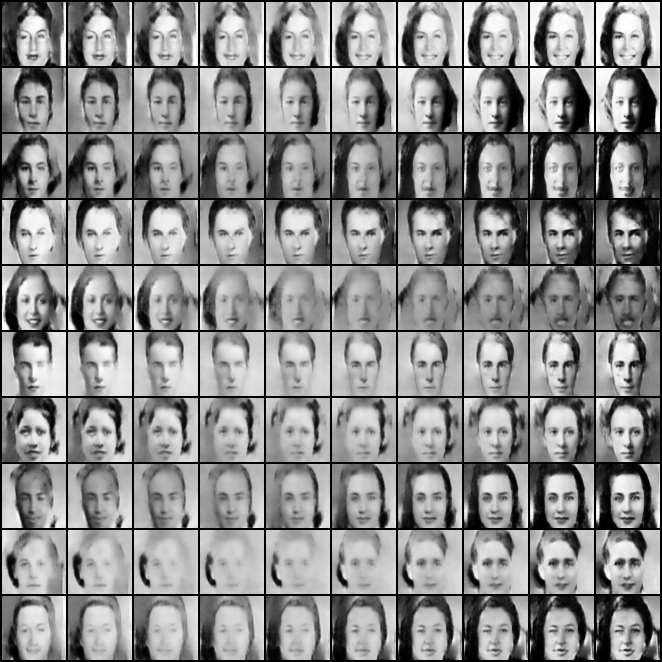} } \label{fig:yearbookBinaryMOMinterp}}
  \caption{Example generated yearbook images from two time periods: Old (1925-1944) and Recent (2000-2013). Target distribution is 83\%/17\% while the given data $\ModFunc\P$ is split 50\%/50\%. Each time period contains enough images to train C-CDGAN successfully. However, the other methods produce qualitatively sharper images (a)--(d) with smoother latent interpolations (e)--(h).}%
 \label{fig:yearbookBinaryCGAN}
\end{figure}

\subsection{MNIST}
Analogous to the class rebalancing problem of Section~\ref{sec:synth}, Figure \ref{fig:multi_boost} shows good performance going from a balanced distribution to specific boosted levels.

Analogous to the self-normalized example of Section~\ref{sec:selfnorm}, we use our self-normalized estimator to manipulate the distribution over twos from the MNIST dataset, where we aim to have fewer curly twos and more twos with a flat bottom. As before, 200 were manually labeled with weights. Fig.~\ref{fig:mnist_twos_data_sorted} shows 100 real images, sorted in terms of their inferred weight. Fig.~\ref{fig:mnist_twos_gens_sorted} shows 100 generated simulations, sorted in the same manner, clearly showing a decrease in the proportion of curly twos. Fig.~\ref{fig:mnist_sevens_lines} shows the inferred weights for both real and simulated data. 

\begin{figure}[t!]
\centering
    \subfloat[Source, even distribution of 0s, 1s, and 5s]{{\includegraphics[width=.27\columnwidth]{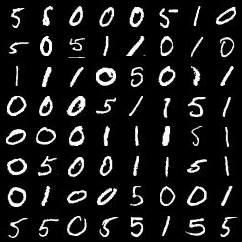} }\label{fig:multi_boost_x_source}}
  \hspace{.05in}
    \subfloat[Source (left), simulation (right); target of $10\%$-$30\%$-$60\%$]{{\includegraphics[width=.27\columnwidth]{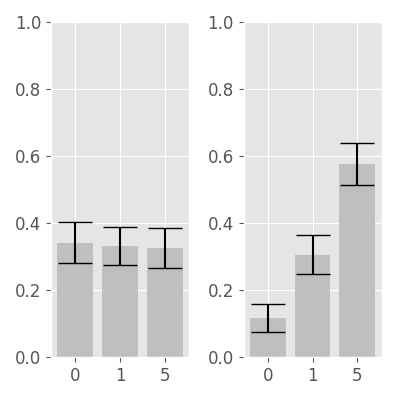} } \label{fig:multi_boost_plots}}
  \hspace{.05in}
    \subfloat[Simulations, boosted distribution]{{\includegraphics[width=.27\columnwidth]{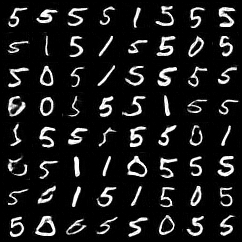} }\label{fig:multi_boost_result}}
  \caption{Importance weights are used to accurately boost an even class distribution to specified levels.}%
 \label{fig:multi_boost}
\end{figure}

\begin{figure}[h]
  \centering
  \subfloat[Data]{{\includegraphics[width=.27\columnwidth]{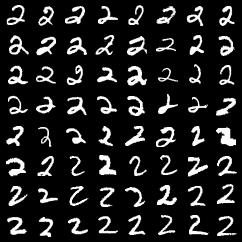} } \label{fig:mnist_twos_data_sorted}}
  \subfloat[Generator]
  {\hspace{0.05in}
  \includegraphics[width=.27\columnwidth]{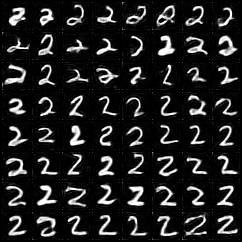} \label{fig:mnist_twos_gens_sorted}}
 {\hspace{0.05in}
 \subfloat[KS distance]{
  \includegraphics[width=.27\columnwidth]{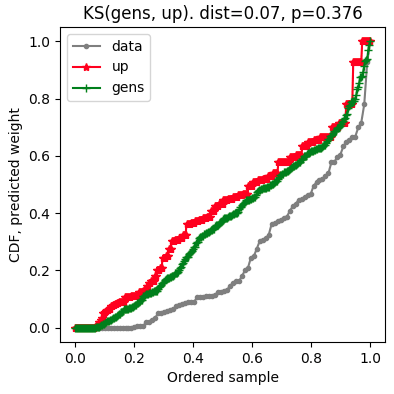} \label{fig:mnist_twos_lines}}}\\
  \caption{A small set of labels are used to train an importance weighted estimator that aims to boost the presence of flat-bottomed twos. In \ref{fig:mnist_twos_data_sorted} and \ref{fig:mnist_twos_gens_sorted}, samples are sorted by predicted weight, and in \ref{fig:mnist_twos_lines}, the empirical CDFs of data, generated, and importance duplicated draws, are shown, where the latter serves as a theoretical target. The  generated distribution produces more flat-bottomed twos, and is close in distance to the target, with $d_{KS} = 0.07$, $p=0.376$.}%
 \label{fig:mnist_twos}
\end{figure}

\end{appendices}

\end{document}